\documentclass[11pt]{article}
\usepackage[margin=1in]{geometry}
\usepackage[utf8]{inputenc} 
\usepackage[T1]{fontenc}    
\usepackage{hyperref}       
\usepackage{url}            
\usepackage{booktabs}       
\usepackage{amsfonts}       
\usepackage{nicefrac}       
\usepackage{microtype}      
\usepackage{amsmath,amssymb,algorithm,algorithmicx,algpseudocode,amsthm,bm}
\usepackage{adjustbox}
\usepackage{graphicx}
\usepackage{latexsym}
\usepackage{listings}
\usepackage{xcolor}
\usepackage{placeins}
\usepackage{multirow}
\usepackage{tabu}
\usepackage{authblk}
\usepackage[resetlabels]{multibib}

\newcites{ltex}{References}

\title{Efficient Private ERM for Smooth Objectives}

\newtheorem{assumption}{Assumption}
\newtheorem{theorem}{Theorem}
\newtheorem{definition}{Definition}
\newtheorem{lemma}{Lemma}
\newtheorem*{proske}{Proof Sketch}

\makeatletter
\newcommand{\thline}{%
    \noalign {\ifnum 0=`}\fi \hrule height 1pt
    \futurelet \reserved@a \@xhline
}
\newcolumntype{"}{@{\hskip\tabcolsep\vrule width 1pt\hskip\tabcolsep}}
\makeatother

\author[1]{Jiaqi Zhang\thanks{zjqgr@126.com}}
\author[1]{Kai Zheng\thanks{zhengk92@pku.edu.cn}}
\author[1]{Wenlong Mou\thanks{mouwenlong@pku.edu.cn}}
\author[1]{Liwei Wang\thanks{wanglw@cis.pku.edu.cn}}
\affil[1]{School of EECS, Peking University}
\begin{document}

\maketitle

\begin{abstract}
    In this paper, we consider efficient differentially private empirical risk minimization from the viewpoint of optimization algorithms. For strongly convex and smooth objectives, we prove that gradient descent with output perturbation not only achieves nearly optimal utility, but also significantly improves the running time of previous state-of-the-art private optimization algorithms, for both $\epsilon$-DP and $(\epsilon, \delta)$-DP. For non-convex but smooth objectives, we propose an RRPSGD (Random Round Private Stochastic Gradient Descent) algorithm, which provably converges to a stationary point with privacy guarantee. Besides the expected utility bounds, we also provide guarantees in high probability form. Experiments demonstrate that our algorithm consistently outperforms existing method in both utility and running time.  
\end{abstract}

\section{Introduction}

Data privacy has been a central concern in statistics and machine learning, especially when utilizing sensitive data such as financial accounts and health-care data. Thus, it is important to design machine learning algorithms which protect users' privacy. As a rigorous and standard concept of privacy, differential privacy \cite{dwork2006calibrating} guarantees that the algorithm learns statistical information of the population, but nothing about individual users. In the framework of differential privacy, there has been a long line of research studying differentially private machine learning algorithms, such as \cite{chaudhuri2011differentially,chaudhuri2013near,lei2011differentially,rubinstein2012learning,talwar2015nearly}.

Among all machine learning models, empirical risk minimization (ERM) plays an important role, as it covers a variety of machine learning tasks. Once we know how to do ERM privately, it is straightforward to obtain differentially private algorithms for a large variety of machine learning problems, such as classification, regression, etc. The earliest representative work of this research line is done by Chaudhuri et al. \cite{chaudhuri2011differentially}. They proposed two approaches to guarantee differential privacy of the output of ERM, namely, output perturbation and objective perturbation. Output perturbation is a variant of Laplace (Gaussian) mechanism, where the stability of exact solutions plays a key role in the analysis. Objective perturbation is done by adding noise to ERM objective and solving precise solution to the new problem. In Kifer et at. \cite{kifer2012private}, they extend the method of objective perturbation, and prove similar results for more general case, especially for high-dimensional learning. 

Both \cite{chaudhuri2011differentially} and \cite{kifer2012private} were discussed in terms of precise solutions to optimization problems. In reality, however, it is not only intractable but also unnecessary to obtain precise solutions. Instead, we always use some optimization algorithms to obtain approximate solutions. In this context, the interaction between privacy-preserving mechanisms and optimization algorithms has non-trivial implications to both sides: running the algorithm for finite cycles of iteration inherently enhances stability; on the other hand, noise added to preserve privacy introduces new challenges to the convergence rate of optimization algorithms. The purpose of this research is therefore two-fold: both utility and time complexity are of central concern.

\begin{table*}[b]
\centering
  \begin{adjustbox}{max width = 1\textwidth}
  \begin{tabular}{|l|c|c|c|c|}
    \hline
     \multirow{2}*{ }& \multicolumn{2}{c|}{Ours} & \multicolumn{2}{c|}{Bassily et al. \cite{bassily2014private} }  \\
     \cline{2-5}
     & Utility & Runtime  & Utility & Runtime \\
     \hline
     $\mu$-S.C., $\epsilon$-DP  & $\mathcal{O}(\frac{d^2}{n^2\varepsilon^2})$ & $\mathcal{O}(nd\log(\frac{n\varepsilon}{d}))$  & $\mathcal{O}(\frac{\log(n)d^2}{n^2\varepsilon^2})$   & $\approx \mathcal{O}(n^3d^3 \min\{1,\varepsilon n,d\log(dn)\})$\\
     \hline
     $\mu$-S.C., $(\epsilon, \delta)$-DP & $\mathcal{O}(\frac{d\log(1/\delta)}{n^2\varepsilon^2})$ & $\mathcal{O}(nd\log(\frac{n\varepsilon}{\sqrt{d\log(\delta)}}))$ & $\mathcal{O}(\frac{d\log^3(n/\delta)}{n^2\varepsilon^2})$ & $\mathcal{O}(n^2d)$   \\
     \hline
     Nonconvex &  $\mathcal{O}(\frac{\sqrt{d}}{n\epsilon} \log \frac{n}{\delta})$ & $\mathcal{O}(n^2d)$ &\multicolumn{2}{c|}{NA} \\
     \hline
  \end{tabular}
  \end{adjustbox}
  \caption{Comparison with existing results (S.C. means strongly convex)}
  \label{comparison1}
\end{table*}

In literature, \cite{bassily2014private} and \cite{song2013stochastic} use stochastic gradient descent (SGD) as the basic optimization algorithm to solve ERM, and add noise to each iteration to achieve $(\varepsilon, \delta)$-differential privacy. Bassily et al. \cite{bassily2014private} develop an efficient implementation of exponential mechanism to achieve $\varepsilon$-differential privacy. Furthermore, they also prove their algorithms match the lower bounds for corresponding problems (ignoring $\log$ factors). Nearly at the same time of this paper, \cite{wu2017bolt} combined output perturbation with permutation SGD according to its stability analysis. However, their utility results only hold for constant number of passes over data, which did not match existing lower bound. Besides these worst-case results, \cite{talwar2014private} gives a more careful analysis based on constraint set geometry, which leads to better utility bounds in specific problems such as LASSO. Despite the success of previous works in terms of utility, there are still much work to do from a practical perspective.
\begin{itemize}
   \item[1.] Both of algorithms proposed in \cite{bassily2014private} and \cite{talwar2014private} have to run at least $\Omega(n^2)$ iterations to reach the ideal accuracy ($n$ is number of data points), which is much slower than non-private version and makes the algorithm impractical for large data sets. Can we do faster while still guarantee privacy and accuracy?
   \item[2.] Note that all existing results only hold for convex ERM, yet non-convex objective functions have been increasingly important, especially in deep neural networks. Can we design an efficient and private optimization algorithms for non-convex ERM with theoretical guarantee?
\end{itemize} 

Fortunately, the answers to above questions are both "yes". In this paper, we will give two efficient algorithms with privacy and utility guarantees. Throughout this paper, we assume the objective function is $\beta$-smooth (See Section \ref{preli} for precise definition), which is a natural assumption in optimization and machine learning. Smoothness allows our algorithm to take much more aggressive gradient steps and converge much faster, which is not fully utilized in previous work like \cite{bassily2014private} and \cite{talwar2014private}. Moreover, smoothness also makes it possible for non-convex case to have theoretical guarantees around stationary points.


 Technically, our work is partially inspired by the work of Hardt et al. \cite{hardt2015train}, in which they established the expected stability $\mathbb{E}\|\mathcal{A}(S)-\mathcal{A}(S')\|$ of SGD ($\mathcal{A}$ is a randomized algorithm, and $S, S'$ are neighboring datasets). Using similar techniques we can derive worst case stability for deterministic algorithms like classical gradient descent, which plays a core role in private algorithm design. For non-convex ERM, we use a variant of Randomized Stochastic Gradient (RSG) algorithm in \cite{ghadimi2013stochastic} to achieve privacy and accuracy at the same time. Our contributions can be summarized as follows:
 
 
\begin{itemize}
  \item[1.] In strongly convex case, by choosing appropriate learning rate, basic gradient descent with output perturbation not only runs much faster than private SGD \cite{bassily2014private}, but also improves its utility by a logarithmic factor, which matches the lower bound in \cite{bassily2014private} \footnote{Here we only consider the performance in terms of $n$ and $d$, which is main concern in learning theory, and regard strongly convex parameter $\mu$ as a constant.}. Besides, we also show its generalization performance.
  \item[2.] We propose a private optimization algorithm for non-convex function, and prove its utility, both in expectation form and high probability form;
  \item[3.] Numerical experiments show that our algorithms consistently outperform existing approaches.
\end{itemize}

In the following, we will give a detailed comparison of our results to existing approaches.

 \textbf{Comparison with existing results} As the closest work to ours is Bassily et al. \cite{bassily2014private}, and their algorithms also match the lower bound in terms of utility, we mainly compare our results with theirs.  Results are summarized in Table \ref{comparison1} (Notations are defined in the next section).
    
 From Table \ref{comparison1}, we can see that our algorithm significantly improves the running time for strongly convex objectives, and achieves slightly better utility guarantee with a $\log$ factor. For non-convex functions, our result is the first differentially private algorithm with theoretical guarantee in this case, to the best of our knowledge.

\section{Preliminaries}
\label{preli}

In this section, we provide necessary background for our analyses, including differential privacy and basic assumptions in convex optimization.

\subsection{Setting}

Throughout this paper, we consider differentially private solutions to the following ERM problem:
\[
  \begin{split}
    \min_{w\in\mathbb{R}^d} & \quad F(w,S):=\frac{1}{n}\sum_{i=1}^n f(w,\xi_i) \\
  \end{split}
\]
where $S=\{(x_1,y_1),\dots,(x_n,y_n)\}, \xi_i=(x_i,y_i)$ is training set, and $\hat{w} := \arg\min_{w} F(w,S)$. The loss function $f$  usually satisfies $f\geq 0$ and we use $f(\cdot)$ to represent $f(\cdot,\xi_i)$ for simplicity.  
\begin{assumption}
  $f(\cdot,\xi_i)$ is $\beta$-smooth, i.e
  \[
  |f(u)-f(v)-\langle\nabla f(v),u-v\rangle| \leq \frac{\beta}{2}\|u-v\|^2
  \]
  If, in addition, $f(\cdot,\xi_i)$ is convex, then above equation reduced to
  \[
  f(u)-f(v)-\langle\nabla f(v),u-v\rangle \leq \frac{\beta}{2}\|u-v\|^2
  \]
\end{assumption}
Actually, $\beta$-smoothness is a common assumption as in \cite{nesterov2013introductory}.

\subsection{Differential Privacy}

Let $S$ be a database containing $n$ data points in the data universe $\mathcal{X}$. Then two databases $S$ and $S'$ are said to be neighbors, if $|S| = |S'| = n$, and they differ in exactly one data point. The concept of differential privacy is defined as follows:
\begin{definition}{(Differential privacy \cite{dwork2006calibrating})}
A randomized algorithm $\mathcal{A}$ that maps input database into some range $\mathcal{R}$ is said to preserve $(\varepsilon, \delta)$-differential privacy, if for all pairs of neighboring databases $S, S'$ and for any subset $A \subset \mathcal{R}$, it holds that
\[
  Pr(\mathcal{A}(S)\in A) \leq Pr(\mathcal{A}(S')\in A)e^\varepsilon + \delta.
\]
In particular, if $\mathcal{A}$ preserves $(\varepsilon,0)$-differential privacy, we say $\mathcal{A}$ is $\varepsilon$-differentially private.
\end{definition}

The core concepte and fundamental tools used in DP are sensitivity and Gaussian (or Laplace) mechanism, introduced as follows:
\begin{definition}{($L_2$-sensitivity)}
The $L_2$-sensitivity of a deterministic query $q(\cdot)$  is defined as 
\[
\Delta_2(q) = \sup_{S,S'} \|q(S) - q(S')\|_2
\]
\end{definition}
Similarly, we can defined $L_1$ sensitivity as $\Delta_1(q) = \sup_{S,S'} \|q(S) - q(S')\|_1$. 
\begin{lemma}{(Laplace and Gaussian Mechanism \cite{dwork2014algorithmic} )} \label{privacy lemma}
  Given any function $q: \mathcal{X}^n \rightarrow R^k$, the Laplace mechanism is defined as :
  \[
    \mathcal{M}_{L} (S, q(\cdot), \epsilon) = q(S) + (Y_1, \dots, Y_k)
  \]
  where $Y_i$ are i.i.d random variables drawn from $\mathrm{Lap}(\Delta_1(q)/\epsilon)$. This mechanism preserves $\varepsilon$-differential privacy. Similarly, for Gaussian mechanism, each $Y_i$ are $i.i.d$ drawn from $\mathcal{N}(0, \sigma^2)$, and let $\sigma = \sqrt{2\ln(1.25/\delta)}\Delta_2(q)/\epsilon$. Gaussian mechanism preserves $(\varepsilon, \delta)$-differential privacy.
\end{lemma}

\section{Main Results}

In this section, we present our differentially private algorithms and analyze their utility for strongly convex, general convex and non-convex cases respectively. Because of the limitation of space, we only give proof skecthes of some critical results. For detailed proof, please see the full version of this paper.

\subsection{Convex case} 
\label{sub:convex_case}

We begin our results with the assumption that each $f$ is $\mu$-strongly convex. Our algorithm is a kind of output perturbation mechanism which is similar to Chaudhuri's \cite{chaudhuri2011differentially}, but we do not assume an exact minimizer can be accessed. With strong convexity and smoothness, which are the most common assumptions in machine learning, our algorithm runs significantly faster than Bassily et al. \cite{bassily2014private}, and matches their lower bounds for utility. Furthermore, the number of iterations needed in our algorithm is significantly less than previous approaches, making it scalable with large amount of data. From a practical perspective, our algorithm can achieve both $\varepsilon$-DP and $(\varepsilon,\delta)$-DP by simply adding Laplacian and Gaussian noise respectively. 

  
\begin{algorithm}[!]
\caption{Output Perturbation Full Gradient Descent}\label{NPFGD}
  \begin{algorithmic}[1]
  \renewcommand{\algorithmicrequire}{\textbf{Input:}}
  \renewcommand{\algorithmicensure}{\textbf{Output:}}
  
  \Require $S=\{(x_1,y_1),\dots,(x_n,y_n)\}$, convex loss function $f(\cdot,\cdot)$(with Lipschitz constant $L$), number of iteration $T$, privacy parameters $(\varepsilon,\delta)$, $\eta$, $\Delta$, $w_0$
  \For {$t=0$ to $T-1$}
    \State $w_{t+1} := w_t - \frac{\eta}{n}\sum_{i=1}^n \nabla f(w_t,\xi_i)$
  \EndFor 
  \If {$\delta=0$}
    \State sample $z \sim P(z) \propto \exp(-\frac{\varepsilon\|z\|_2}{\Delta})$ \Comment{(This is for $\epsilon$-DP)}
  \Else
    \State sample $z \sim P(z) \propto \exp(-\frac{\varepsilon^2\|z\|_2^2}{4\log(2/\delta)\Delta^2})$ \Comment{(This is for $(\epsilon, \delta)$-DP)}
  \EndIf
  \Ensure $w_{priv} = w_T + z$
  \end{algorithmic}
\end{algorithm}

As sensitivity serves as an essential technique in the differential privacy analysis, to start with, we will prove the sensitivity of gradient descent based on the idea of Hardt et al. \cite{hardt2015train}. Let $\Delta_T=\Vert w_T-w_T' \Vert_2$ be the $L_2$-sensitivity of an algorithm, where $w_T$ and $w_T'$ are the variables in $T$-th round, for two neighboring databases $S$ and $S'$ respectively. 

\begin{lemma}\label{lemma_con_stab}
Assume $f(\cdot)$ is convex, $\beta$-smooth and $L$-Lipschitz. If we run gradient descent(GD) algorithm with constant step size $\eta \leq \frac{1}{\beta}$ for T steps, then the $L_2$-sensitivity of GD satisfies
\[
\Delta_T \leq \frac{3LT\eta}{ n}
\]
\end{lemma}
\begin{proske}
  According te preperties of convex and smooth, one can deduce the following recursion inequalities:
  \begin{align*}
    \Delta_{t+1}^2 \leqslant \Delta_t^2 + \frac{4\eta L}{n}\Delta_t + \frac{8\eta^2 L^2}{n^2}
  \end{align*}
  Then using an induction argument, one obtain the conclusion.
\end{proske}
With the same idea, one can prove following stability results for strongly convex and smooth functions.
\begin{lemma}\label{lemma_str_con_stab}
Assume $f(\cdot)$ is $\mu$-strongly convex, $\beta$-smooth and $L$-Lipschitz. If we run gradient descent(GD) algorithm with constant step size $\eta \leq \frac{1}{\beta+\mu}$ for T steps, then the $L_2$-sensitivity of GD satisfies
\[
\Delta_T \leq \frac{5L(\mu+\beta)}{n\mu\beta}
\]
\end{lemma}

\begin{theorem}\label{OPFGD_privacy}
  Algorithm \ref{NPFGD} is $(\varepsilon,\delta)$-differential private for any $\varepsilon > 0$ and $\delta \in [0,1)$, with concrete setting in below theorems.
\end{theorem}
\begin{theorem}\label{OPFGD_utility_sc}
If $f(\cdot)$ is $\mu$-strongly convex, $\beta$-smooth. Assume $\|\hat{w}\|\leq D$ and $f(\cdot)$ is $L$-Lipschitz for all $\{w:\|w\|\leq 2D\}$. Let $\eta = \frac{1}{\mu+\beta}$ and $\Delta = \frac{5L(1+\beta/\mu)}{n\beta}$, For $w_{priv}$ output by Algorithm \ref{NPFGD}, we have the following.
\begin{enumerate}
    \item For $\varepsilon$-differential privacy, if we set $T = \Theta\left(\left[\frac{\mu^2+\beta^2}{\mu\beta}\log(\frac{\mu^2n^2\varepsilon^2D^2}{L^2d^2})\right]\right)$. Then,
    \[
      \mathbb{E}\ F(w_{priv},S) - F(\hat{w},S)  \leqslant O\left(\frac{\beta L^2d^2}{n^2\varepsilon^2\mu^2}\right)
    \]
    \item For $(\varepsilon,\delta)$-differential privacy, if we set $T = \Theta\left(\left[\frac{\mu^2+\beta^2}{\mu\beta}\log(\frac{\mu^2n^2\varepsilon^2D^2}{L^2d\log(1/\delta)})\right]\right)$. Then,
    \[
      \mathbb{E}\ F(w_{priv},S) - F(\hat{w},S) \leqslant O\left(\frac{\beta L^2d\log(1/\delta)}{n^2\varepsilon^2\mu^2}\right)
    \]
\end{enumerate}
\end{theorem}
\begin{proske}
  For $\epsilon$-DP, we have 
  \begin{align*}
    &\mathbb{E} F(w_{priv},S) - F(\hat{w},S) \\
    \leq & \mathbb{E} \left[ F(w_T,S) + \langle \nabla F(w_T,S), z\rangle + \frac{\beta}{2}\|z\|^2 \right] - F(\hat{w},S) \\
    = &\left( F(w_T,S) - F(\hat{w},S) \right) + \frac{\beta}{2}\mathbb{E}\|z\|^2 \\
    \leqslant & \frac{\beta}{2}\exp\left(-\frac{2\mu\beta T}{(\mu+\beta)^2}\right)D^2 + \frac{25L^2(\mu+\beta)^2(d+1)d}{n^2\varepsilon^2\mu^2\beta}
  \end{align*}
  where the last inequality comes from exponential convergence rate of GD and the magnitude of noise, which is closely related to stability results. Thus we obatin desired utility guarantees with above optimal choice of $T$. The proof of $(\epsilon, \delta)$-DP is exactly the same.
\end{proske}

It is worth noticing that the results of Bassily et al. \cite{bassily2014private}, hold without smoothness assumption, but their method does not improve too much even with this assumption. This is because they use an SGD-based algorithm, where smoothness could not help in the convergence rate, and where step sizes have to be set conservatively. For strongly convex functions, smoothness assumption is necessary when we use a perturbation-based algorithm. As a result, the running time of our algorithm is of the order given in table \ref{comparison1} , because each iteration requires $O(nd)$ computation. On the other hand, a function can become very steep without this assumption, so adding noise to the result of gradient method may cause an unbounded error to the function value. 

For the generalization ability of our algorithm, we assume all examples $\xi_i$ are i.i.d drawn from the unknown distribution $\mathcal{D}$, and $w^*$ is the minimizer of population risk $G(w) = \mathbb{E}_{\xi} f(w, \xi)$. Define excess risk of any $w$ as $\mathrm{ExcessRisk}(w):= G(w) - G(w^*)$. Here we only discuss excess risk of $(\varepsilon, \delta)$-differential privacy algorithm, for $\varepsilon$-differential privacy algorithm, the approach is the same. 

The most usual technique to obtain excess risk is to use Theorem 5 and inequality (18) in \cite{shalev2009stochastic}. In this case, we assume loss function $f(w, \xi)$ is $\mu$-strongly convex and $L$-Lipschitz continuous (w.r.t $w$) within a ball of radius $R$, which includes the population minimizer $w^*$. 
Thus, by substituting our utility bound in Theorem \ref{OPFGD_utility_sc}, we can obtain: with probability at least $1-\gamma$, $\mathrm{ExcessRisk}(w_{priv}) \leqslant \tilde{\mathcal{O}}(\frac{L\sqrt{\beta d}}{n\epsilon \mu \gamma})$\footnote{Note $\frac{1}{\gamma}$ dependence on failure probability $\gamma$ can be improved to $\log \frac{1}{\gamma}$ by boosting the confidence method used in \cite{shalev2010learnability} } ($\tilde{\mathcal{O}}$ means we ignore all $\log$ factors). Another method to obtain excess risk is to directly use the relation between the stability of gradient descent and its excess risk, as shown in \cite{hardt2015train}. Then we have: 
\begin{align*}
  \mathrm{ExcessRisk}(w_{priv}) & =  G(w_{priv}) - G(w_T) + G(w_T) - G(w^*) \\
  & \leqslant L\| z\| + \mathrm{Error_{opt}}(w_T) + L\Delta_T
\end{align*}
where $\mathrm{Error_{opt}}(w_T)$ represents the empirical optimization error. Note $\|z\|$ term in above inequality can be bounded through tail bound of $\chi^2$ distribution, hence, it will lead to nearly same excess risk bound as the first method. 

If we remove the strong convexity property of our loss function, we have the following theoretical guarantee of Algorithm \ref{NPFGD}.
\begin{theorem}\label{OPFGD_utility}
If $f(\cdot)$ is $L$-Lipschitz, convex and $\beta$-smooth on $\mathbb{R}^d$. Assume $\|\hat{w}\|\leq D$ and let $\eta = \frac{1}{\beta}$ and $\Delta = \frac{3LT}{\beta n}$, then for $w_{priv}$ output by Algorithm \ref{NPFGD}, we have the following.
  \begin{enumerate}
    \item For $\varepsilon$-differential privacy, if we set $T = \Theta \left(\left[\frac{\beta^2n^2\varepsilon^2D^2}{L^2d^2}\right]^\frac{1}{3}\right)$, then,
    \[
    \mathbb{E}\ F(w_{priv},S) - F(\hat{w},S) \leqslant O\left(\left[\frac{\sqrt{\beta}Ld\|\hat{w}\|^2}{n\varepsilon}\right]^\frac{2}{3}\right)
    \]
    \item For $(\varepsilon,\delta)$-differential privacy, if we set $T = \Theta\left(\left[\frac{\beta^2n^2\varepsilon^2D^2}{L^2d\log(1/\delta)}\right]^\frac{1}{3}\right)$ then,
    \[
    \mathbb{E} F(w_{priv},S) - F(\hat{w},S)  \leqslant O\left(\left[\frac{L\sqrt{\beta d\log(1/\delta)}\|\hat{w}\|^2}{n\varepsilon}\right]^\frac{2}{3}\right)
    \]
  \end{enumerate}
\end{theorem}
Though the utility guarantee is weaker than Bassily et al. \cite{bassily2014private} in general convex case by a factor of $O(\frac{1}{ \sqrt[3]{n} })$, but when $d$ is smaller than $n$,  then both bounds are below the typical $\tilde{\Theta}(n^{-\frac{1}{2}})$ generalization error in learning theory.\footnote{Actually without any other assumption, the performances of almost all private algorithms have polynomial dependence over $d$, which will hurt generalization error in some degree for large $d$.} So our algorithm does not harm accuracy of machine learning task indeed. Furthermore, compared with \cite{bassily2014private}, our algorithm runs uniformly faster for pure $\epsilon$-DP, and also faster for $(\epsilon,\delta)$-DP for high-dimensional problems. This acceleration is mainly due to smoothness of objective function. Moreover, our experimental results show that our algorithm is significantly better than \cite{bassily2014private} under both convex and strongly convex settings, in the sense that our algorithm not only achieves a lower empirical error but also runs faster than theirs (See Section \ref{sec:experimental_results} for more details). As for generalization property for general convex loss, we can solve it along the same road as strongly convex case by adding a regularization term $\frac{\mu}{2} \|w\|_2^2$ (where $\mu = \frac{\sqrt{2}L^{1/2}(\beta d)^{1/4}}{\sqrt{n\epsilon \gamma}R}$). Therefore, in convex case, we can obtain: with probability at least $1-\gamma$, $\mathrm{ExcessRisk}(w_{priv}) \leqslant \tilde{\mathcal{O}}(\frac{RL^{1/2}(\beta d)^{1/4}}{\sqrt{n\epsilon \gamma}})$.

\subsection{Nonconvex case} 
\label{sec:nonconvex_case}

In this section, we propose a random round private SGD which is similar with private SGD in \cite{bassily2014private}. We will show that our algorithm can differential privately (we only focus on $(\varepsilon,\delta)$-DP this time) find a stationary point in expectation with diminishing error. To the best of our knowledge, this is the first theoretical result about differentially private non-convex optimization problem and this algorithm also achieve same utility bound with \cite{bassily2014private}, which are known to be near optimal for more restrictive convex case. Our algorithm is inspired by the work of Bassily et al. \cite{bassily2014private} and Ghadimi et al. \cite{ghadimi2013stochastic}.

  
\begin{algorithm}[h]
\caption{Random Round Private Stochastic Gradient Descent}\label{DPSGD}
  \begin{algorithmic}[1]
  \renewcommand{\algorithmicrequire}{\textbf{Input:}}
  \renewcommand{\algorithmicensure}{\textbf{Output:}}
  
  \Require $S=\{(x_1,y_1),\dots,(x_n,y_n)\}$, loss function $f(\cdot,\cdot)$ (with Lipschitz constant $L$), privacy parameters $(\varepsilon,\delta)(\delta > 0)$, a probability distribution $\mathbb{P}$ (See distribution setting in the Theorem \ref{NSGD_utility}) over $[n^2]$, learning rate $\{\eta_k\}$
  \State draw $R$ from $\mathbb{P}$
  \For {$t=0$ to $R-1$}
    \State sample $\xi \sim U(S)$
    \State sample $z_t \sim  \exp(-\frac{\varepsilon^2\|z\|_2^2}{8L^2\log(3n/\delta)\log(2/\delta)})$
    \State $w_{t+1} := w_t - \eta_t ( \nabla f(w_t,\xi) + z_t ) $
  \EndFor 
  \Ensure $w_{priv} = w_R$
  \end{algorithmic}
\end{algorithm}

Note our iteration times $R$ satisfies $R \leq n^2$, so the same argument with bassily et al. \cite{bassily2014private} can be applied to ensure the DP property of Algorithm \ref{DPSGD}. The technical details for proofs are deferred to appendix. The utility guarantee mainly comes from the convergence result of SGD (Ghadimi et al. \cite{ghadimi2013stochastic}) under non-convex setting.
\begin{theorem}{(Privacy guarantee)}\label{NSGD_privacy}
  Algorithm \ref{DPSGD} is $(\varepsilon,\delta)$ differential private for any $\varepsilon \in (0,1]$ and $\delta \in (0,1)$.
\end{theorem}

\begin{theorem}{(Utility guarantee)}\label{NSGD_utility}
  If $f(\cdot)$ is $L$-Lipschitz and $\beta$-smooth, and we choose $\mathbb{P}$ which satisfies 
  \begin{align*}
    \mathbb{P}(k+1) := & Pr(R=k+1) = \frac{2\eta_k-\beta\eta_k^2}{\sum_{r=0}^{n^2-1} 2\eta_r-\beta\eta_r^2}\\
    & \textbf{for} \quad k = 0,1,\dots,n^2-1.
  \end{align*}
  Assume $\eta_k$ are chosen such that $\eta_k<\frac{2}{\beta}$. Let $\sigma^2 = 4L^2 + \frac{4dL^2\log(3n/\delta)\log(2/\delta)}{\varepsilon^2}$, then for $w_{priv}$ output by Algorithm \ref{DPSGD}, we have the following (the expectation is taken w.r.t $\mathbb{P}$ and $\xi_i$ )
  \[
    \mathbb{E}\|\nabla F(w_{priv},S)\|^2 \leq \frac{\beta[D_F^2 + \sigma^2\sum_{k=0}^{n^2-1} \eta_k^2]}{\sum_{r=0}^{n^2-1} 2\eta_r-\beta\eta_r^2}
  \]
  where $D_F = \sqrt{2(F(w_0,S)-F^*)/\beta}$
  and $F^*$ is a global minimum of $F$, note that $F^* \geq 0$ in our settings. 

  What's more, if we take $\eta_k := \min\{\frac{1}{\beta},\frac{D_F}{\sigma n}\}$ then we get,
      \[
      \mathbb{E}\|\nabla F(w_{priv},S)\|^2 = O\left(\frac{\beta L\sqrt{d\log(n/\delta)\log(1/\delta)}D_F}{n\varepsilon}\right)
      \]
  If in addition, $f(\cdot)$ is convex and $\|\hat{w}\| \leq D$, then we have,
    \[
      \mathbb{E}\ F(w_{priv},S) - F(\hat{w},S) = O\left(\frac{L\sqrt{d\log(n/\delta)\log(1/\delta)}D}{n\varepsilon}\right)
    \]
\end{theorem}
\begin{proske}
  Let $G(w_t) = \nabla f(w_t,\xi) + z_t$. Note that over the randomness of $\xi$ and $z_t$, we have $\mathbb{E}G(w_t) = \nabla F(w_t,S)$ and $\mathbb{E}\|G(w_t)-\nabla F(w_t,S)\|^2 \leq 4L^2 + \frac{8L^2\log(3n/\delta)\log(2/\delta)}{\varepsilon^2}$. Thus the theorem holds immediately based on convergence results of \cite{ghadimi2013stochastic}.
\end{proske}
As in convex and strongly convex cases, we are using output perturbation to protect privacy, so it is straightforward to obtain high probability version of this bound based on tail bounds for Laplacian and Gaussian distribution. Thus we only consider high probability bounds for non-convex case. The following lemma serves as an important tool for our high-probability analysis.


\begin{lemma}{\cite{lan2012validation}} \label{martingale}
    Let $X_1, \dots, X_T$ be a martingale difference sequence, i.e., $\mathbb{E}_{t-1}[X_t] = 0$ (where $\mathbb{E}_{t-1}[\cdot]$ denotes the expectation conditioned on all the randomness till time $t-1$) for all $t$. Suppose that for some values $\sigma_t$, for $t= 1,2, \dots, T$, we have $\mathbb{E}_{t-1}[\mathrm{exp}(\frac{X_t^2}{\sigma_t^2})] \leqslant \mathrm{exp}(1)$. Then with probability at least $1-\delta$, we have 
    \begin{equation*}
      \sum_{t=1}^TX_t \leqslant \sqrt{3\log(\frac{1}{\delta})\sum_{t=1}^T\sigma_t^2}
    \end{equation*}
\end{lemma}  
Now, we can proceed to prove the following theorem about high probability bound.

\begin{theorem}\label{hpb_nc}
  When in the same condition of Theorem \ref{NSGD_utility}, by setting $\eta_k := \min\{\frac{1}{\beta},\frac{D_F}{\sigma n}\}$, then with probability at least $1 - \gamma$ (Note this probability is over the noise and the randomness of choosing point in each round), there is 
  \begin{equation*}
    \mathbb{E}\|\nabla F(w_{priv},S)\|^2 \leqslant  O\left(\frac{\sqrt{d\log(1/\gamma)\log(n/\delta)\log(1/\delta)}}{n\varepsilon}\right)
  \end{equation*}
\end{theorem}

\section{Experimental Results} 
\label{sec:experimental_results}
To show the effectiveness of our algorithm in real world data, we experimentally compare our algorithm with Bassily et al. \cite{bassily2014private} for convex and strongly convex loss function. To be more specific, we consider (regularized) logistic regression on 3 UCI \cite{Lichman:2013} binary classification datasets and (regularized) Huber regression on 2 UCI regression datasets (see Table \ref{dataset_info} for more details\footnote{Note all category variables in these datasets are translated into binary features.}).

\begin{table}[h]
\centering
\begin{adjustbox}{max width = 0.4\textwidth}
  \begin{tabular}{c|c|c|c}
    \hline
            & $n$ & $d$ & type \\
    \hline
      BANK  &  45211 & 42 & classification\\
    \hline
      ADULT &  32561 & 110 & classification\\
    \hline
      CreditCard & 30000 & 34 & classification\\
    \hline
      WINE & 6497 & 12 & regression \\
    \hline
      BIKE & 17379 & 62 & regression \\
    \hline
  \end{tabular}
  \end{adjustbox}
  \vspace{3pt}
\caption{Dataset information}\label{dataset_info}
\end{table}
\begin{table*}[t]
\small
\centering
  \begin{adjustbox}{max width = 0.6\textwidth}
  \begin{tabu}{|c|[1pt]c|[1pt]c|[1pt]c|c|[1pt]c|c| [1pt]}
  \thline
  \multirow{2}{*}{Dataset} & \multirow{2}{*}{$\mu$}  & \multirow{2}{*}{$\varepsilon$} & \multicolumn{2}{c|[1pt]}{Error} & \multicolumn{2}{c|[1pt]}{Runtime(CPU time)} \\ \tabucline{4-7}
                           &                         &                                 & ours,$(\varepsilon,\delta)$ & Bassily,$(\varepsilon,\delta)$ &  ours,$(\varepsilon,\delta)$ & Bassily,$(\varepsilon,\delta)$ \\ \tabucline[1pt]{1-7}

  \multirow{8}{*}{BANK}  & \multirow{4}{*}{0} &  0.1   & $\bm{0.3983}$ & 2.2552  & $\bm{12.613}$ & 518.67 \\ \cline{3-7}
                         &                    &  0.5   & $\bm{0.2231}$ & 1.4585  & $\bm{36.796}$ & 519.33 \\ \cline{3-7}
                         &                    &  1     & $\bm{0.1459}$ & 1.0203  & $\bm{58.305}$ & 519.02 \\ \cline{3-7}
                         &                    &  2     & $\bm{0.0838}$ & 0.7824  & $\bm{92.501}$ & 518.27 \\ \cline{2-7}
                         &\multirow{4}{*}{0.1}&  0.1   & $\bm{0.2566}$ & 0.4829  & $\bm{20.483}$ & 518.03 \\ \cline{3-7}
                         &                    &  0.5   & $\bm{0.0106}$ & 0.4090  & $\bm{40.541}$ & 519.44 \\ \cline{3-7}
                         &                    &  1     & $\bm{0.0025}$ & 0.3387  & $\bm{49.311}$ & 516.73 \\ \cline{3-7}
                         &                    &  2     & $\bm{0.0005}$ & 0.2475  & $\bm{57.947}$ & 520.17 \\ \tabucline[1pt]{1-7}

  \multirow{8}{*}{ADULT} & \multirow{4}{*}{0} &  0.1   & $\bm{0.0499}$ & 0.6229  & $\bm{23.813}$ & 250.50 \\ \cline{3-7}
                         &                    &  0.5   & $\bm{0.0208}$ & 0.6081  & $\bm{69.536}$ & 254.14 \\ \cline{3-7}
                         &                    &  1     & $\bm{0.0122}$ & 0.4781 & $\bm{110.20}$ & 254.18 \\ \cline{3-7}
                         &                    &  2     & $\bm{0.0065}$ & 0.3691   & $\bm{175.01}$ & 253.72 \\ \cline{2-7}
                         &\multirow{4}{*}{0.1}&  0.1   & $\bm{3.2039}$ & 5.2166  & $\bm{112.09}$ & 256.70 \\ \cline{3-7}
                         &                    &  0.5   & $\bm{0.1287}$ & 5.1532  & $\bm{193.98}$ & 255.36 \\ \cline{3-7}
                         &                    &  1     & $\bm{0.0309}$ & 5.1148  & $\bm{229.23}$ & 255.69 \\ \cline{3-7}
                         &                    &  2     & $\bm{0.0080}$ & 5.1009  & 264.23        &$\bm{257.23}$\\ \tabucline[1pt]{1-7}
  \multirow{8}{*}{CreditCard}& \multirow{4}{*}{0} &  0.1   & $\bm{0.0293}$ & 0.4106  & $\bm{4.9595}$ & 190.30 \\ \cline{3-7}
                                                        
                         &                    &  0.5   & $\bm{0.0102}$ & 0.4220  & $\bm{14.591}$ & 190.89 \\ \cline{3-7}
                         &                    &  1     & $\bm{0.0053}$ & 0.3140  & $\bm{22.983}$ & 188.67 \\ \cline{3-7}
                         &                    &  2     & $\bm{0.0024}$ & 0.2708  & $\bm{36.721}$ & 188.86 \\ \cline{2-7}
                         &\multirow{4}{*}{0.1}&  0.1   & $\bm{0.3643}$ & 1.3271  & $\bm{13.664}$ & 190.36 \\ \cline{3-7}
                         &                    &  0.5   & $\bm{0.0141}$ & 1.2973  & $\bm{22.012}$ & 189.97 \\ \cline{3-7}
                         &                    &  1     & $\bm{0.0035}$ & 1.2792  & $\bm{25.743}$ & 188.81 \\ \cline{3-7}
                         &                    &  2     & $\bm{0.0008}$ & 1.2501  & $\bm{29.256}$ & 187.97 \\ \tabucline[1pt]{1-7}

  \multirow{8}{*}{WINE}  & \multirow{4}{*}{0} &  0.1   & $\bm{0.6061}$ & 6.1755  & $\bm{0.1672}$ & 6.3859 \\ \cline{3-7}
                         &                    &  0.5   & $\bm{0.2487}$ & 4.1900  & $\bm{0.4328}$ & 6.3828 \\ \cline{3-7}
                         &                    &  1     & $\bm{0.1713}$ & 3.0972  & $\bm{0.7469}$ & 6.4234 \\ \cline{3-7}
                         &                    &  2     & $\bm{0.1110}$ & 1.3609  & $\bm{1.1719}$ & 6.3016 \\ \cline{2-7}
                         &\multirow{4}{*}{0.5}&  0.1   & $\bm{1.0842}$ & 8.2900  & $\bm{0.0922}$ & 6.4328 \\ \cline{3-7}
                         &                    &  0.5   & $\bm{0.0364}$ & 7.9584  & $\bm{0.1437}$ & 6.3625 \\ \cline{3-7}
                         &                    &  1     & $\bm{0.0101}$ & 6.5471  & $\bm{0.1891}$ & 6.5391 \\ \cline{3-7}
                         &                    &  2     & $\bm{0.0024}$ & 5.3811  & $\bm{0.1812}$ & 6.4484 \\ \tabucline[1pt]{1-7}

  \multirow{8}{*}{BIKE}  & \multirow{4}{*}{0} &  0.1   & $\bm{5.4659}$ & 35.279  & $\bm{0.1531}$ & 6.4953 \\ \cline{3-7}
                         &                    &  0.5   & $\bm{4.0404}$ & 30.822  & $\bm{0.4375}$ & 6.2375 \\ \cline{3-7}
                         &                    &  1     & $\bm{3.2768}$ & 27.196  & $\bm{0.6922}$ & 6.2734 \\ \cline{3-7}
                         &                    &  2     & $\bm{2.4081}$ & 23.865  & $\bm{1.1766}$ & 6.3969 \\ \cline{2-7}
                         &\multirow{4}{*}{0.5}&  0.1   & $\bm{0.0555}$ & 3.0770  & $\bm{0.1031}$ & 6.5766 \\ \cline{3-7}
                         &                    &  0.5   & $\bm{0.0301}$ & 3.0448  & $\bm{0.1578}$ & 6.5094 \\ \cline{3-7}
                         &                    &  1     & $\bm{0.0242}$ & 2.1792  & $\bm{0.1625}$ & 6.4094 \\ \cline{3-7}
                         &                    &  2     & $\bm{0.0232}$ & 1.0406  & $\bm{0.1984}$ & 6.3625 \\ \tabucline[1pt]{1-7}
  \end{tabu}
  \end{adjustbox}
  \vspace{3pt}
  \caption{Summary of experimental results}\label{exp_results}
\end{table*}
The loss function for logistic regression is $f(w,\xi) = \log(1+\exp(1+y\langle w,x\rangle))$. And for Huber regression, the loss function $f(w,\xi;\delta) = h_\delta(\langle w, x\rangle - y)$, where \footnote{For loss functions in above problems, we add an additional square regularization term with parameter $\mu$ to make them strongly convex.} 
\[
h_\delta(u) = 
\begin{cases}
\frac{1}{2}u^2 & for \quad |u|\leq \delta, \\
\delta(|u|-\frac{1}{2}\delta) & otherwise.  
\end{cases}
\]

All parameters are chosen as stated in theorems in both papers, except that we use a mini-batch version of SGD in \cite{bassily2014private} with batch size $m=50$, since their algorithm in its original version requires prohibitive $n^2$ time of iterations for real data, which is too slow to run. This conversion is a natural implication of amplification lemma, which preserves the same order of privacy and affects utility with constant ratio. We evaluate the minimization error $\mathbb{E} F(w_{priv},S) - F(\hat{w},S)$ and running time of these algorithms under different $\varepsilon = \{0.1,0.5,1,2\}$ and $\delta = 0.001$. $\mu/2 |w|^2$ term is added to make the objective mu-strongly convex. We choose $\mu=0$ to show our methods also have better performance in convex setting, and choose $\mu =0.1, 0.5$ to compare under different degree of strong convexity.The experimental results are averaged over 100 independent rounds. Table \ref{exp_results} illustrates the experimental results of both methods.

From Table \ref{exp_results}, we can see our algorithm outperforms existing one on both utility and runtime under almost all settings. 

\section{Conclusion} 
\label{sec:conclusion}
We study differentially private ERM for smooth loss function under (strongly) convex and non-convex situation. Though output perturbation has been well studied before, our results show that adding noise to approximate solutions instead of exact solutions has important implications to both privacy and running time. Our work is inspired by \cite{hardt2015train}, whose technique for stability analysis of SGD can be applied to deterministic gradient descent algorithms. We show that for strongly convex and smooth objectives, our output perturbation gradient descent achieves optimal utility and runs much faster than the existing private SGD in Bassily et al. \cite{bassily2014private}. And for general convex objectives, it is also an efficient practical algorithm due to its fast convergence and reasonable utility. From the experimental results, our algorithm achieves lower optimization error and runtime in almost all cases compared to private SGD. For non-convex objectives, by carefully chosen parameters, we show that a random rounds private SGD can reach a stationary point in expectation. This is first theoretical bound for differentially private non-convex optimization to the best of our knowledge. 




\clearpage
\bibliographystyle{plain}    
\bibliography{DPref.bib}

\setcounter{section}{0}
\appendix
\renewcommand{\appendixname}{Appendix~\Alph{section}}
\section{Appendix} 
\label{sec:appendix}
\subsection{Proof of Theorem \ref{OPFGD_privacy}} 
\label{ssub:proof_of_ref_npfgd_privacy}
\begin{proof}{(Proof of Theorem \ref{OPFGD_privacy})}
  The theorem follows directly by combining Lemma \ref{lemma_con_stab}, Lemma \ref{lemma_str_con_stab}, Lemma \ref{privacy lemma}, so we only need to prove Lemma \ref{lemma_con_stab} and Lemma \ref{lemma_str_con_stab}
\end{proof}

\begin{proof}{(Proof of Lemma \ref{lemma_con_stab})}
Without loss of generality, assume databases $S$ and $S'$ only differ on sample $\xi_n$. let $\Delta_t:=\|w_t-w_t'\|$ and we have that $\Delta_0 = 0$. We use $f_i(\cdot)$ to denote $f(\cdot,\xi_i)$ for simplicity. Using the fact that $\frac{1}{\beta}\|\nabla f(w_t) - \nabla f(w_t')\|^2 \leq \langle w_t-w_t',\nabla f(w_t) - \nabla f(w_t')\rangle$, we have for any $\eta \leq \frac{1}{\beta}$,
\[
\begin{split}
  \Delta_{t+1}^2 &= \|w_{t+1}-w_{t+1}'\|^2 \\ 
               &= \|w_t-\eta\nabla F(w_t,S) - w_t' + \eta\nabla F(w_t',S')\|^2 \\
               &= \|w_t-w_t'\|^2 - 2\eta\langle w_t-w_t',\nabla F(w_t,S) - \nabla F(w_t',S')\rangle + \eta^2 \|\nabla F(w_t,S) - \nabla F(w_t',S')\| \\
               &\leq \|w_t-w_t'\|^2 - 2\eta \langle w_t - w_t', \frac{1}{n}\sum_{i=1}^{n-1} (\nabla f_i(w_t) - \nabla f_i(w_t')) \rangle + 2\eta^2 \|\frac{1}{n}\sum_{i=1}^{n-1}(\nabla f_i(w_t) - \nabla f_i(w_t'))\|^2 \\
               & \qquad - 2\eta \langle w_t - w_t', \frac{1}{n} (\nabla f_n(w_t) - \nabla f_n'(w_t')) \rangle + 2\eta^2 \|\frac{1}{n}(\nabla f_n(w_t) - \nabla f_n'(w_t'))\|^2 \\
               &\leq \|w_t - w_t'\|^2 - \left(\frac{2\eta}{\beta}-2\eta^2\right) \|\frac{1}{n}\sum_{i=1}^{n-1}(\nabla f_i(w_t) - \nabla f_i(w_t'))\|^2\\
               & \qquad - 2\eta \langle w_t - w_t', \frac{1}{n} (\nabla f_n(w_t) - \nabla f_n'(w_t')) \rangle + 2\eta^2 \|\frac{1}{n}(\nabla f_n(w_t) - \nabla f_n'(w_t'))\|^2 \\
               &\leq \Delta_t^2 + \frac{4\eta L}{n}\Delta_t + \frac{8\eta^2 L^2}{n^2}
\end{split}
\]
Due to the fact that $\Delta_0 = 0$, we have $\Delta_t \leq \frac{3Lt\eta}{n}$ for $t=1$. The result now follows from a simple induction argument that suppose $\Delta_t \leq \frac{3Lt\eta}{n}$ for some $t$, then
\[
\begin{split}
  \Delta_{t+1}^2 &\leq \frac{9L^2t^2\eta^2}{n^2} + \frac{12L^2t\eta^2}{n^2} + \frac{8L^2\eta^2}{n^2}   \\
                 &= \frac{L^2\eta^2}{n^2}(9t^2+12t+8) \\
                 &\leq \frac{9L^2\eta^2}{n^2}(t+1)^2
\end{split}
\]
\end{proof}
\begin{proof}{(Proof of Lemma \ref{lemma_str_con_stab})}
Assume databases $S$ and $S'$ only differ on sample $\xi_n$. Using a similar approach, let $\Delta_t:=\|w_t-w_t'\|$ and we have that $\Delta_0 = 0$. We use $f_i(\cdot)$ to denote $f(\cdot,\xi_i)$ for simplicity. Using the fact that $\frac{1}{\beta+\mu}\|\nabla f(w_t) - \nabla f(w_t')\|^2 + \frac{\mu\beta}{\beta+\mu}\|w_t-w_t'\|^2 \leq \langle w_t-w_t',\nabla f(w_t) - \nabla f(w_t')\rangle$,  For any $n\geq 2$, we have for any $\eta \leq \frac{1}{\mu+\beta}$,
\[
\begin{split}
  \Delta_{t+1}^2 &= \|w_{t+1}-w_{t+1}'\|^2 \\ 
               &= \|w_t-\eta\nabla F(w_t,S) - w_t' + \eta\nabla F(w_t',S')\|^2 \\
               &= \|w_t-w_t'\|^2 - 2\eta\langle w_t-w_t',\nabla F(w_t,S) - \nabla F(w_t',S')\rangle + \eta^2 \|\nabla F(w_t,S) - \nabla F(w_t',S')\|^2 \\
               &\leq \|w_t-w_t'\|^2 - 2\eta \langle w_t - w_t', \frac{1}{n}\sum_{i=1}^{n-1} (\nabla f_i(w_t) - \nabla f_i(w_t')) \rangle + 2\eta^2 \|\frac{1}{n}\sum_{i=1}^{n-1}(\nabla f_i(w_t) - \nabla f_i(w_t'))\|^2 \\
               & \qquad - 2\eta \langle w_t - w_t', \frac{1}{n} (\nabla f_n(w_t) - \nabla f_n'(w_t')) \rangle + 2\eta^2 \|\frac{1}{n}(\nabla f_n(w_t) - \nabla f_n'(w_t'))\|^2 \\
               &\leq \left[1- \frac{2(n-1)\eta\mu\beta}{n(\mu+\beta)} \right]\|w_t - w_t'\|^2 - \left(\frac{2\eta}{\mu+\beta}-2\eta^2\right) \|\frac{1}{n}\sum_{i=1}^{n-1}(\nabla f_i(w_t) - \nabla f_i(w_t'))\|^2\\
               & \qquad - 2\eta \langle w_t - w_t', \frac{1}{n} (\nabla f_n(w_t) - \nabla f_n'(w_t')) \rangle + 2\eta^2 \|\frac{1}{n}(\nabla f_n(w_t) - \nabla f_n'(w_t'))\|^2 \\
               &= \left[1- \frac{2(n-1)\eta\mu\beta}{n(\mu+\beta)} \right]\Delta_t^2 + \frac{4\eta L}{n}\Delta_t + \frac{8\eta^2 L^2}{n^2}
\end{split}
\]
 In the above inequality, it is easy to see both $\{w_t\}$ and $\{w_t'\}$ are in the ball $\{w:\|w\| \leqslant 2D\}$, so we can use $L$-Lipschitz property. Due to the fact that $\Delta_0 = 0$, we have $\Delta_t \leq \frac{5L}{n(\mu+\beta)} \leq \frac{5L}{n(\mu+\beta)\frac{\mu\beta}{(\mu+\beta)^2}} = \frac{5L(\mu+\beta)}{n\mu\beta}$ for $t=1$. The result now follows from a simple induction argument that suppose $\Delta_t \leq \frac{5L(\mu+\beta)}{n\mu\beta}$ for some $t$, then
\[
\begin{split}
    \Delta_{t+1}^2 &\leq \left[1- \frac{2(n-1)\eta\mu\beta}{n(\mu+\beta)} \right]\left(\frac{5L(\mu+\beta)}{n\mu\beta}\right)^2 + \frac{20\eta L^2(\mu+\beta)}{n^2\mu\beta} + \frac{8L^2\eta^2}{n^2} \\
                  &\leq \left(\frac{5L(\mu+\beta)}{n\mu\beta}\right)^2 - \frac{\eta\mu\beta}{\mu+\beta}\left(\frac{5L(\mu+\beta)}{n\mu\beta}\right)^2 + \frac{20\eta L^2(\mu+\beta)}{n^2\mu\beta} + \frac{8L^2\eta^2}{n^2} \\
                  &= \left(\frac{5L(\mu+\beta)}{n\mu\beta}\right)^2 - \frac{25\eta L^2(\mu+\beta)}{n^2\mu\beta} + \frac{20\eta L^2(\mu+\beta)}{n^2\mu\beta} + \frac{8\eta^2L^2}{n^2} \\
                  &\leq \left(\frac{5L(\mu+\beta)}{n\mu\beta}\right)^2 + \frac{\eta L^2}{n^2}\left(-\frac{5(\mu+\beta)}{\mu\beta}+8\eta\right) \\
                  &\leq \left(\frac{5L(\mu+\beta)}{n\mu\beta}\right)^2
\end{split}
\]
\end{proof}

\subsection{Proof of Theorem \ref{OPFGD_utility_sc}} 
\label{sub:proof_of_theorem_ref_opfgd_utility_sc}

\begin{lemma}{\citeltex{nesterov2013introductory}}\label{convergence_sc}
Assume that loss function $f(\cdot)$ is $\mu$-strongly convex and $\beta$-smooth. If we run gradient descent (GD) algorithm with constant step size $\eta \leq \frac{2}{\beta+\mu}$ for T steps, then
\[
F(w_T,S) - F(\hat{w},S) \leq \frac{\beta}{2}\exp\left(-\frac{2\eta\mu\beta T}{\mu+\beta}\right)\|w_0-\hat{w}\|^2
\]
\end{lemma}
\begin{lemma}{\label{d_gamma_var}}
For a random variable $z\in\mathbb{R}^d$ which satisfies $z \sim \exp(-\frac{\|z\|_2}{\sigma})$, then
\[
\mathbb{E} \|z\|^2 = d(d+1)\sigma^2
\]
\end{lemma}

\begin{lemma}{\label{d_gaussian_var}}
For a random variable $z\in\mathbb{R}^d$ which satisfies $z \sim \exp(-\frac{\|z\|^2_2}{2 \sigma^2})$, then
\[
\mathbb{E} \|z\|^2 = d\sigma^2
\]
\end{lemma}

\begin{proof}{(Proof of Theorem \ref{OPFGD_utility_sc})}
Combine Lemma \ref{convergence_sc}, \ref{d_gamma_var} and \ref{d_gaussian_var}, recall that $w_{priv} = w_T + z$. By $\beta$-smoothness of $F$.\\
\[
\begin{split}
  \mathbb{E} F(w_{priv},S) - F(\hat{w},S) &\leq \mathbb{E} \left[ F(w_T,S) + \langle \nabla F(w_T,S), z\rangle + \frac{\beta}{2}\|z\|^2 \right] - F(\hat{w},S) \\
                                          &= \left( F(w_T,S) - F(\hat{w},S) \right) + \frac{\beta}{2}\mathbb{E}\|z\|^2    
\end{split}
\]
For $\varepsilon$-differential privacy, by setting $T = \Theta\left(\left[\frac{\mu^2+\beta^2}{\mu\beta}\log(\frac{\mu^2n^2\varepsilon^2D^2}{L^2d^2})\right]\right)$
\[
\begin{split}
\mathbb{E} F(w_{priv},S) - F(\hat{w},S) &\leq \frac{\beta}{2}\exp\left(-\frac{2\mu\beta T}{(\mu+\beta)^2}\right)D^2 + \frac{25L^2(\mu+\beta)^2(d+1)d}{n^2\varepsilon^2\mu^2\beta} \\
                                        &\leq O\left(\frac{\beta L^2d^2}{n^2\varepsilon^2\mu^2}\right)
\end{split}
\]
For $(\varepsilon,\delta)$-differential privacy, by setting $T = \Theta\left(\left[\frac{\mu^2+\beta^2}{\mu\beta}\log(\frac{\mu^2n^2\varepsilon^2D^2}{L^2d\log(1/\delta)})\right]\right)$
\[
\begin{split}
\mathbb{E} F(w_{priv},S) - F(\hat{w},S) &\leq \frac{\beta}{2}\exp\left(-\frac{2\mu\beta T}{(\mu+\beta)^2}\right)D^2 + \frac{50L^2(\mu+\beta)^2(d+1)d}{n^2\varepsilon^2\mu^2\beta} \\
                                        &\leq O\left(\frac{\beta L^2d\log(1/\delta)}{n^2\varepsilon^2\mu^2}\right)
\end{split}
\]

\end{proof}




\subsection{Proof of Theorem \ref{OPFGD_utility}} 
\label{sub:proof_of_theorem_ref_opfgd_utility}

\begin{lemma}{\citeltex{nesterov2013introductory}\label{convergence_convex}}
Assume that loss function $f(\cdot)$ is convex and $\beta$-smooth. If we run gradient descent(GD) algorithm with constant step size $\eta = \frac{1}{\beta}$ for T steps, then
\[
F(w_T,S) - F(\hat{w},S) \leq \frac{2\beta\|w_0 - \hat{w}\|^2}{T}
\]
\end{lemma}

\begin{proof}{(Proof of Theorem \ref{OPFGD_utility})}
Follow the same lines of the proof of Theorem \ref{OPFGD_utility_sc}. Combine Lemma \ref{d_gamma_var}, \ref{d_gaussian_var} and \ref{convergence_convex}, recall that $w_{priv} = w_T + z$. By $\beta$-smoothness of $F$.
\[
\begin{split}
  \mathbb{E} F(w_{priv},S) - F(\hat{w},S) &\leq \mathbb{E} \left[ F(w_T,S) + \langle \nabla F(w_T,S), z\rangle + \frac{\beta}{2}\|z\|^2 \right] - F(\hat{w},S) \\
                                          &= \left(F(w_T,S) - F(\hat{w},S)\right) + \frac{\beta}{2}\mathbb{E}\|z\|^2 
\end{split}
\]
In both cases, the first term is $\frac{2\beta D^2}{T}$, while the second term is $\frac{9T^2L^2d(d+1)}{2\beta n^2\varepsilon^2}$ for $\varepsilon$-DP and changes into $\frac{9T^2L^2d\log(2/\delta)}{\beta n^2\varepsilon^2}$ for $(\varepsilon,\delta)$-DP. Then the theorem holds by setting $T = \Theta(\left[\frac{\beta^2n^2\varepsilon^2D^2}{L^2d^2}\right]^\frac{1}{3})$, $\Theta(\left[\frac{\beta^2n^2\varepsilon^2D^2}{L^2d\log(1/\delta)}\right]^\frac{1}{3})$ respectively.

\end{proof}

\subsection{Proof of Theorem \ref{NSGD_privacy} } 
\label{sub:proof_of_theorem_ref_nsgd_privacy}
Note the fact that $R\leq n^2$, The theorem holds by applying same claims as which used in the Theorem 2.1 of \cite{bassily2014private2}. Here we give the details.
\begin{proof}{(Proof of Theorem \ref{NSGD_privacy})}
  Fix the randomness of $R$ and $\xi_i$, for any $t \leq R$, let $X_t(S) = \nabla f(w_t,\xi) + z_t$ be a random variable whose randomness comes from $z_t$ and conditioned on $w_t$. Let $p_{X_t(S)}(y)$ be the probability measure of $X_t(S)$ induced on $y\in\mathbb{R}^d$. Then for any two neighboring dataset $S$ and $S'$, define the privacy loss random variable \cite{dwork2010boosting} as $C_t = |\log\frac{p_{X_t(S)}(X_t(S))}{p_{X_t(S')}(X_t(S))}|$. By \cite{kifer2012private2}, we have that with probability $1-  \frac{\delta}{2n}$, $C_t \leq \frac{\varepsilon}{2\sqrt{2\log(\frac{2}{\delta})}}$ for all $t\leq R$. Applying Lemma \ref{amplification} with $\alpha = \frac{1}{n}$, we ensure that with probability at least $1 - \frac{\delta}{2n^2}$, $C_t \leq \frac{\varepsilon}{n\sqrt{2\log(\frac{2}{\delta})}}$. The theorem follows from applying Lemma \ref{strong_composition} with $\delta' = \frac{\delta}{2}$ and $T = R \leq n^2$.
\end{proof}
\begin{lemma}{(Amplification \cite{beimel2014bounds})}\label{amplification}
For any dataset $S$ with $|S|=n$, running an $(\varepsilon,\delta)$-differentially private algorithm on uniformly random $\alpha n$ entries of $S$ ensures $(2\alpha\varepsilon,\alpha\delta)$-differential privacy.
\end{lemma}
\begin{lemma}{(Strong composition \cite{dwork2010boosting})}\label{strong_composition}
  For any $\varepsilon>0$, $\delta\geq 0$, $\delta'>0$, an $(\varepsilon,\delta)$-differentially private algorithm preserves $(\varepsilon',T\delta + \delta')$-differential privacy under $T$-fold adaptive composition with $\varepsilon' = \sqrt{2T\log(1/\delta')\varepsilon} + T\varepsilon(e^{\varepsilon}-1)$.
\end{lemma}

\subsection{Proof of Theorem \ref{NSGD_utility}} 
\label{sub:proof_of_theorem_ref_nsgd_utility}
\begin{proof}{(Proof of Theorem \ref{NSGD_utility})}
  Let $G(w_t) = \nabla f(w_t,\xi) + z_t$. Note that over the randomness of $\xi$ and $z_t$, we have $\mathbb{E}G(w_t) = \nabla F(w_t,S)$ and $\mathbb{E}\|G(w_t)-\nabla F(w_t,S)\|^2 \leq 4L^2 + \frac{8L^2\log(3n/\delta)\log(2/\delta)}{\varepsilon^2}$. Thus the theorem holds immediately after applying Lemma \ref{nonconvex_convergence} and Lemma \ref{nonconvex_convergence_corollary} with $T = n^2$.
\end{proof}
\begin{lemma}{(Theorem 2.1 of \cite{ghadimi2013stochastic2})}\label{nonconvex_convergence}
Let $\{\eta_t\}$ be a set of stepsizes which satisfies $\eta_t < \frac{2}{\beta}$ and $T\in \mathbb{N}$. Let $R\in[T]$ be a random variable and 
  \[
  \mathbb{P}(k) := Pr(R=k) = \frac{2\eta_k-L\eta_k^2}{\sum_{k=0}^{T-1} 2\eta_k-\beta\eta_k^2}.
  \]
consider a $R$-round SGD $w_{t+1} = w_t - \eta_tG(w_t)$, where $G(w_t)$ is a stochastic gradient return by some stochastic first order oracle which satisfies $\mathbb{E} G(w_t) = \nabla F(w_t,S)$ and $\mathbb{E} \|G(w_t) - F(w_t,S)\|^2 \leq \sigma^2$.
\begin{enumerate}
  \item for any $T\geq 1$, the following holds
  \[
  \mathbb{E} \|\nabla F(w_R,S)\|^2 \leq \frac{D_f^2 + \sigma^2\sum_{t=1}^T\eta_t^2}{\sum_{t=1}^T(2\eta_t - \beta\eta_t^2)},
  \]
  where $D_f:=\sqrt{\frac{2(F(w_0,S)-F^*)}{\beta}}$ and $F^*$ is the global minimum of $F$.
  \item In addition, if $f(\cdot)$ is convex, then the following holds 
  \[
  \mathbb{E} F(w_R,S) - F(\hat{w},S) \leq \frac{\|w_0-\hat{w}\|^2 + \sigma^2\sum_{t=1}^T\eta^2}{\sum_{t=1}^T(2\eta_t-\beta\eta_t^2)}
  \]
\end{enumerate}
where the expectation is taken with respect to $R$ and the randomness of $G$.
\end{lemma}

\begin{lemma}{(Corollary 2.2 of \citeltex{ghadimi2013stochastic2})}\label{nonconvex_convergence_corollary}
Following the Lemma \ref{nonconvex_convergence}, If the stepsizes are set to $\eta_t :=\min \left\{ \frac{1}{\beta},\frac{D_f}{\sigma\sqrt{T}}\right\}, t=1,\dots,T$. then,
\[
  \mathbb{E} \|\nabla F(w_R,S)\|^2 \leq \frac{\beta D_f^2}{T} + \frac{2D_f\sigma}{\sqrt{T}}.
\]
If $f(\cdot)$ is convex, then we have
\[
  \mathbb{E} F(w_R,S) - F(\hat{w},S) \leq \frac{\beta \|w_0-\hat{w}\|^2}{T} + \frac{2\|w_0-\hat{w}\|\sigma}{\sqrt{T}}.
\]
\end{lemma}


\subsection{Discussion of high probability bounds} 

\begin{proof}{(Proof of Theorem \ref{hpb_nc})}
  Let $\delta_t := \nabla f(w_t,\xi) - \nabla F(w_t,S)$, and denote $F(w) = F(w, S)$ for simplicity. Note $\|\delta_t\| \leqslant 2L$ because of Lipschitz condition. Then according to the definition of $\beta$-smooth and iteration form, there is 
  \begin{align*}
    F(w_{t+1}) = & F(w_t - \eta_t ( \nabla f(w_t,\xi) + z_t )) \\
     \leqslant & F(w_t) - \eta_t \nabla F(w_t)^T (\delta_t + \nabla F(w_t) + z_t) + \frac{\beta}{2} \eta_t^2 \|\delta_t + \nabla F(w_t) + z_t\|^2 \\
     = & F(w_t) - (\eta_t - \frac{\beta}{2} \eta_t^2) \| \nabla F(w_t) \| ^2 - (\eta_t - \beta \eta_t^2) \nabla F(w_t)^T (\delta_t + z_t) + \\
     & \beta \eta_t^2 \delta_t^T z_t + \frac{\beta}{2} \eta_t^2 (\| \delta_t \|^2 + \| z_t \|^2)
  \end{align*}
  Set random variable $X_t := (\beta \eta_t^2 - \eta_t ) \nabla F(w_t)^T \delta_t, Y_t :=  [(\beta \eta_t^2 - \eta_t ) \nabla F(w_t)+ \beta \eta_t^2 \delta_t]^T z_t, R_t := \frac{\beta}{2} \eta_t^2 \| z_t \|^2$. Sum above inequality from $t = 1$ to $n^2$ and rearrange these terms, we obtain:
  \begin{align}
    \sum_{t=1}^{n^2} (\eta_t - \frac{\beta}{2} \eta_t^2) \| \nabla F(w_t) \| ^2 &\leqslant F(w_1) - F(w_{n^2+1}) + \sum_{t=1}^{n^2} (X_t + Y_t + R_t)  + \frac{\beta}{2} \sum_{t=1}^{n^2} \eta_t^2 \| \delta_t \|^2 \nonumber \\
    & \leqslant F(w_1) - F(\hat{w}) + \sum_{t=1}^{n^2} (X_t + Y_t + R_t) + \frac{\beta}{2} \sum_{t=1}^{n^2} \eta_t^2 \| \delta_t \|^2 \label{ineq1}
  \end{align}
  For last term in above inequality, we can bound it by Lipschitz condition:
  \begin{align}
    \frac{\beta}{2} \sum_{t=1}^{n^2}  \eta_t^2 \| \delta_t \| ^2 &\leqslant 2\beta \eta_1^2 n^2 L^2\label{ineq5}
  \end{align}
  Given all the randomness till time $t-1$, we have $\mathbb{E}_{t-1} X_t = 0$, so $X_1, \dots, X_{n^2}$ is a martingale difference, and $\|X_t\|^2 \leqslant 4(\eta_t - \beta \eta_t^2)^2 L^4$. According to martingale inequality Lemma \ref{martingale}, we have with probability at most $\frac{\gamma}{3}$
  \begin{align}
    \sum_{t=1}^{n^2} X_t \geqslant (\eta_1 - \beta \eta_1^2) nL^2 \sqrt{12 \log \frac{3}{\gamma}} \label{ineq2}
  \end{align}
  Similarly for $Y_t$, once given all the randomness till time $t-1$, there is $\mathbb{E}_{t-1} Y_t = 0$. Different with $X_t$, $Y_t$ is unbounded. But luckily, $z_t$ is a multivariate Gaussian random variable with independent components, so it is easy to check: if we set $\sigma_t^2 = 2a(\eta_t L + \beta L \eta_t^2) \alpha^2$, where $a = (1- \mathrm{exp}(-\frac{2}{d}))^{-1} \text{(which is actually $O(d)$) and } \alpha^2 = \frac{4L^2\log(3n/\delta)\log(2/\delta)}{\epsilon^2}$, then we have $\mathbb{E}_{t-1}[\mathrm{exp}(\frac{Y_t^2}{\sigma_t^2})] \leqslant \mathrm{exp}(1)$. Thus using martingale inequality Lemma \ref{martingale}, with probability at most $\frac{\gamma}{3}$, there is
  \begin{align}
    \sum_{t=1}^{n^2} Y_t \geqslant (\eta_1 + \beta \eta_1^2) \frac{nL^2}{\epsilon} \sqrt{24 a\log \frac{3}{\gamma} \log \frac{3n}{\delta} \log \frac{2}{\delta}}    \label{ineq3}
  \end{align} 
  As $R_t$ is a sum of squares of Gaussian random variables, so $\sum_{t=1}^{n^2} R_t$ is actually a scalable $\chi^2$ random variable with $dn^2$ degrees of freedom. According to the tail bound of chi-square distribution \cite{laurent2000adaptive}, with probability at most $\frac{\gamma}{3}$, there is
  \begin{align}
    \sum_{t=1}^{n^2} R_t \geqslant \frac{\beta}{2}\alpha^2\eta_1^2(dn^2 + 2n\sqrt{d\log \frac{3}{\gamma}} + 2\log \frac{3}{\gamma})     \label{ineq4}
  \end{align}  
  Now, combining inequalities (\ref{ineq1}), (\ref{ineq5}), (\ref{ineq2}), (\ref{ineq3}), (\ref{ineq4}), we obtain the theorem.
\end{proof}



\end{document}